\documentclass{article}
\usepackage{spconf,amsmath,graphicx,hyperref}
\usepackage{xcolor}

\usepackage[ruled,vlined]{algorithm2e}
\usepackage{bbm}
\usepackage{amsthm}
\usepackage{amsfonts}
\usepackage{booktabs}
\usepackage{moresize}

\newtheorem{definition}{Definition}

\newtheorem{lemma}{Lemma}
\theoremstyle{definition}
\newtheorem{remark}{Remark}
\newtheorem{theorem}{Theorem}
\newtheorem{assumption}{Assumption}

\newcommand{\fix}[1]{}

\def\approach{SpiKFAX}

\title{On the second-order optimization for spiking neural networks}
\twoauthors
 {Ngoc Phu Doan}
	{Centre for Secure Information Technologies\\Queen's University Belfast
    }
 {Ihsen Alouani}
	{Centre for Secure Information Technologies\\Queen's University Belfast
    }
\begin{document}
%
\maketitle

\begin{abstract}
Spiking Neural Networks (SNNs) offer an energy-efficient alternative to conventional neural networks by exploiting sparse, binary spikes, and event-driven computation. However, the training of SNNs remains challenging, as spiking activations create a sharp loss landscape that hinders training, and diagonal-curvature optimizers such as the Adam family may fail to capture this geometry. The extension of curvature-based optimization methods to SNNs is further complicated by the sparse, discrete, and temporally recurrent nature of their underlying dynamics. To address these limitations, we propose \approach, a second-order optimization method that formulates a computationally tractable, Kronecker-factored approximation of the Fisher information matrix specifically adapted to the structure of SNNs. Empirical evaluation across five architectures and seven datasets demonstrates that \approach~ consistently yields improvements in test accuracy and training stability relative to other popular optimizers.

\end{abstract}
\section{Introduction}
\label{sec:introduction}
Spiking Neural Networks (SNNs) have emerged as an energy-efficient alternative to conventional artificial neural networks (ANNs), owing to their sparse, event-driven computation instead of continuous multiplication operations ~\cite{rathi2021exploring}. However, training SNNs remains challenging: their non-differentiable spiking activation forces the loss landscape into a sharper, more irregular geometry than that of ANNs~\cite{kang2026a2sg,deng2022temporal}. As illustrated in Fig. \ref{fig:loss_landscape}, the loss landscape of a spiking VGG11 (S-VGG11) exhibits markedly higher curvature than its ANN counterpart under the same weight perturbation, making first-order optimizers prone to slow convergence and poor generalization.

Adaptive approaches such as Adam try to mitigate this by maintaining a second-moment estimate of the gradient, which can be interpreted as a diagonal approximation of the curvature~\cite{gomes2025adafisher}. However, this diagonal approximation discards the off-diagonal curvature information that captures interactions between parameters, which is particularly informative in the sharp, spiking-induced loss landscapes described above. As a result, Adam-family optimizers can converge slowly and settle at weaker generalization points when applied to SNNs.

\begin{figure}[t]
    \centering
    \includegraphics[width=\linewidth]{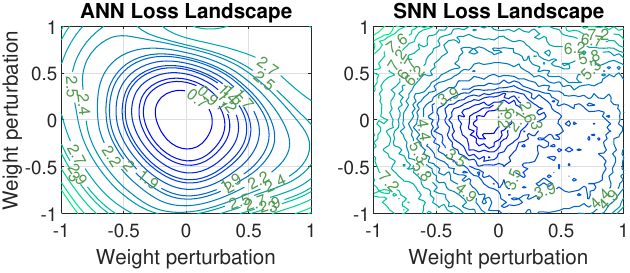}
    \vspace{-6mm}\caption{Loss landscape sharpness between VGG11 and S-VGG11 on the CIFAR10 dataset.}\vspace{-4mm}
    \label{fig:loss_landscape}
\end{figure}
 
A natural remedy is to incorporate second-order curvature information directly into the optimizer. The Hessian matrix, or its Fisher-information surrogate, has been approximated in various ways to precondition gradient updates in deep networks~\cite{gomes2025adafisher,martens2015optimizing}. However, extending these approximations to SNNs is not straightforward: SNNs are sparse in their spike activity, discrete in their output, and time-recurrent in their membrane dynamics, none of which are properties that curvature approximations designed for static, feedforward networks account for. Meanwhile, several optimization algorithms have been proposed specifically for SNNs, targeting complementary aspects of training such as objective-function design, sparsity-enforcing updates, and sharpness-aware minimization to seek flatter minima~\cite{windhager2026linearized,nicholson2026sharpnessawaresurrogatetrainingonsensor,saponati2025feedback}. Yet, to the best of our knowledge, none of these methods exploit curvature information through a Kronecker-factored approximation of the Fisher matrix, leaving a gap between the second-order optimization literature and SNN-specific training methods.

\noindent\textbf{Contributions:}
\begin{itemize}
\item We introduce \textbf{\approach}, an adaptive optimization approach motivated by curvature approximation, that captures second-order information in the sharp loss landscape of SNNs.
\item We propose a computable solution to approximate the Fisher Information Matrix for SNNs, deriving a Kronecker-factored form that accounts for SNNs' time-recurrent, surrogate-gradient-based dynamics while remaining tractable at scale.
\item We conduct extensive experiments across five model architectures and seven datasets, spanning both static and neuromorphic vision benchmarks, demonstrating that \approach~ consistently improves accuracy and training stability over SGD, Adam, and AdamW.
\end{itemize}

\section{Methodology}
\label{sec:method}
\noindent\textbf{Problem formulation.} Given the SNN's parameter $\theta$, the data distribution $\mathcal D$, and the loss function $l$, we minimize $\mathcal{L}(\theta) = \mathbb{E}_{(x,y)\sim\mathcal{D}}[\ell(f_\theta(x),y)]$, a non-convex, non-differentiable objective due to the spiking nonlinearity. We seek an approximate stationary point $\theta$ s.t.
$\mathbb{E}[\|\nabla\mathcal{L}_\sigma(\theta)\|^2]\le\epsilon^2$ via the
preconditioned update $\theta_{k+1}^{(\ell)} = \theta_k^{(\ell)} - \gamma P_k^{(\ell)} g^{(\ell)}(\theta_k)$, where $P_k^{(\ell)}$ is the inverse curvature approximated by our proposed algorithm (c.f. Lemma \ref{lem:fim-approx}), $g$ denotes the gradient.

\noindent\textbf{Spiking Neural Network (SNN).} Given an input spike $x_t$ with timestamp $t\in[1,\tau]$, a spiking layer outputs a sequence of spikes $s_t\in\{0,1\}$. We use the leaky integrate-and-fire neuron with a subtractive (soft) reset:
\begin{equation}
\label{eq:snn_def}
\left\{
\begin{aligned}
u_{t} &= \beta u_{t-1} + Wx_t - V_{thr}s_{t-1} \\
s_t &= \Theta\Big(u_t-V_{thr}\Big)
\end{aligned}
\right.
\end{equation}
where $V_{thr}$ is the firing threshold, $\beta\in[0,1]$ the membrane decay, $W$ the learnable synaptic weight, and $\Theta$ the Heaviside step function, with $u_0=0$ and $s_0=0$.
\fix{The recurrence was written $u_{t+1}=\beta u_t+Wx_t-V_{thr}s_t$, i.e.\ the input $x_t$ entered the membrane one step \emph{after} its own index. Under that convention $\partial\mathcal{L}/\partial W=\sum_t\delta_{t+1}x_t^{\top}$, which contradicted the expression $\sum_t\delta_tx_t^{\top}$ used later and in the algorithm. Re-indexing as above makes the two agree.}

$\Theta$ is non-differentiable, so backpropagation is not directly possible. The standard remedy is the surrogate gradient: $\partial s_t/\partial u_t$ is replaced in the backward pass by $\sigma'(u_t-V_{thr})$ for a differentiable surrogate $\sigma$, e.g.\ superspike~\cite{zenke2018superspike}, arctangent~\cite{stan2024learning}, or sigmoid. All derivatives below are therefore surrogate derivatives, and every $\approx$ that follows from this substitution is written explicitly.

\begin{lemma}[Recurrent Jacobian, surrogate form]
\label{lemma:jacobian}
Under the surrogate substitution, the derivative of $u_t$ with respect to $u_{t-1}$ is
\begin{equation}
    \frac{\partial u_{t}}{\partial u_{t-1}}\approx \beta - V_{thr}\,\sigma'(u_{t-1}-V_{thr}).
\end{equation}
\end{lemma}
\fix{Stated as an exact identity (``The deriviate ... is'') although it is exact only for the true Jacobian of a network whose reset path is kept attached; with $\Theta$ replaced by $\sigma$ it is an approximation, and it must be flagged as such since the whole Fisher construction inherits it.}

\begin{lemma}[Backpropagation through time]
\label{lemma:bptt}
Let $\delta_t:=\partial\mathcal{L}/\partial u_t$. With the terminal condition $\delta_{\tau+1}=0$,
\begin{equation}
\label{equ:backpropagation}
\delta_t =
\underbrace{
\frac{\partial \mathcal{L}}{\partial s_t}
\sigma'(u_t - V_{thr})
}_{\text{direct path (this timestep's output)}}
+
\underbrace{
\delta_{t+1}
\left(
\beta - V_{thr}\sigma'(u_t - V_{thr})
\right)
}_{\text{recurrent path (future timesteps)}}
\end{equation}
where $\partial\mathcal{L}/\partial s_t$ aggregates the contributions of \emph{all} downstream layers that consume $s_t$ at time $t$.
\end{lemma}
\fix{Two omissions: the terminal condition $\delta_{\tau+1}=0$ was missing, so the recursion was not well posed; and $\partial\mathcal{L}/\partial s_t$ was left undefined, which matters because in a multi-layer network it is not the derivative of the output loss alone.}

\begin{proof}
$u_t$ influences $\mathcal{L}$ only through the spike $s_t$ it emits and through the next membrane state $u_{t+1}$, so the chain rule gives $\delta_t=\frac{\partial\mathcal{L}}{\partial s_t}\frac{\partial s_t}{\partial u_t}+\delta_{t+1}\frac{\partial u_{t+1}}{\partial u_t}$. Substituting $\partial s_t/\partial u_t\approx\sigma'(u_t-V_{thr})$ and Lemma~\ref{lemma:jacobian} yields \eqref{equ:backpropagation}. At $t=\tau$ there is no future state, hence $\delta_{\tau+1}=0$.
\end{proof}

Because $W$ is shared across timesteps, $\frac{\partial \mathcal{L}}{\partial W}=\sum_{t=1}^{\tau}\delta_tx_t^\top$. We write $g_t:=\delta_tx_t^\top$ for the per-timestep contribution, so that $\nabla_W\mathcal{L}=\sum_{t=1}^\tau g_t$.

\begin{table*}[t]
    \centering
    \caption{Performance comparison of different optimizers across neuromorphic datasets. All results are reported as mean $\pm$ std.}
    \label{tab:neuromorphic_results}

    \footnotesize
    \setlength{\tabcolsep}{2.4pt}
    \renewcommand{\arraystretch}{1.0}

    \begin{tabular}{lcccccccccccc}
        \toprule

        & \multicolumn{4}{c}{\textbf{N-MNIST}}
        & \multicolumn{4}{c}{\textbf{CIFAR10-DVS}}
        & \multicolumn{4}{c}{\textbf{DVS128 Gesture}} \\

        \cmidrule(lr){2-5}
        \cmidrule(lr){6-9}
        \cmidrule(lr){10-13}

        \textbf{Architecture}
        & \textbf{SGD} & \textbf{Adam} & \textbf{AdamW} & \textbf{\approach}
        & \textbf{SGD} & \textbf{Adam} & \textbf{AdamW} & \textbf{\approach}
        & \textbf{SGD} & \textbf{Adam} & \textbf{AdamW} & \textbf{\approach} \\

        \midrule
        
        S-LeNet5
        & $92.66_{1.27}$ & $93.76_{1.32}$ & $93.97_{0.63}$ & $\textbf{98.24}_{0.11}$
        & $16.6_{0.0}$ & $40.6_{0.6}$ & $42.7_{0.3}$ & $\textbf{46.2}_{1.14}$
        & $50.37_{1.3}$ & $67.42_{0.76}$ & $69.7_{0.38}$ & $\textbf{76.13}_{1.89}$ \\

        S-VGG11
        & $96.06_{0.74}$ & $96.93_{0.55}$ & $97.35_{0.21}$ & $\textbf{99.22}_{0.01}$
        & $50.7_{1.25}$ & $47.9_{3.2}$ & $43.15_{0.25}$ & $\textbf{56.15}_{0.85}$
        & $53.78_{2.27}$ & $52.84_{0.19}$ & $46.4_{0.95}$ & $\textbf{71.78}_{0.57}$ \\

        S-VGG16
        & $97.13_{0.32}$ & $95.57_{0.45}$ & $96.68_{0.12}$ & $\textbf{98.49}_{0.47}$
        & $45.8_{1.0}$ & $29.0_{3.8}$ & $31.95_{10.05}$ & $\textbf{50.5}_{1.3}$
        & $55.11_{0.19}$ & $51.89_{3.79}$ & $54.92_{3.78}$ & $\textbf{67.05}_{0.76}$ \\

        S-ResNet18
        & $50.39_{9.14}$ & $89.32_{1.03}$ & $93.43_{0.89}$ & $\textbf{99.13}_{0.16}$
        & $31.3_{2.8}$ & $29.75_{3.35}$ & $31.7_{0.4}$ & $\textbf{46.95}_{2.75}$
        & $43.18_{0.39}$ & $49.24_{1.5}$ & $48.86_{4.2}$ & $\textbf{70.27}_{2.46}$ \\

        \bottomrule
    \end{tabular}
\end{table*}

\noindent\textbf{SNN-specific Kronecker factorization.} We introduce \textbf{\approach}, a second-order optimizer for SNNs. The idea is to precondition the gradient with the Fisher information matrix, which captures the curvature of the loss landscape. Computing it exactly is intractable: the matrix is $N\times N$ and its inversion costs $O(N^3)$, with $N$ the number of parameters. We therefore build a Kronecker-factored approximation adapted to the temporal structure of SNNs.

\begin{definition}[Fisher information matrix]
\label{def:fisher}
For a model with predictive distribution $p_\theta(y\,|\,x)$,
\begin{equation}
F_W=\mathbb{E}_{x\sim\mathcal{D}}\;\mathbb{E}_{\hat y\sim p_\theta(\cdot|x)}
\Big[\operatorname{vec}(\nabla_W\log p_\theta(\hat y|x))\operatorname{vec}(\nabla_W\log p_\theta(\hat y|x))^{\!\top}\Big].
\end{equation}
Replacing $\hat y\sim p_\theta(\cdot|x)$ by the observed label $y$ yields the \emph{empirical} Fisher, a different matrix that does not converge to the Fisher (nor to the Gauss--Newton matrix) away from a zero-residual optimum.
\end{definition}
\fix{$F_W$ was defined as $\mathbb{E}[(\sum_t\operatorname{vec}g_t)(\sum_t\operatorname{vec}g_t)^{\top}]$ with $g_t$ taken from the \emph{training loss at the true label}. That is the empirical Fisher, not the Fisher, so ``approximating the Fisher information matrix'' was claiming more than the derivation delivered. Either sample $\hat y$ from the model (Alg.~\ref{alg:pseudo-code}) or rename the quantity the empirical Fisher throughout, and say which one the released code computes.}

\begin{definition}[Kronecker product] For $A\in\mathbb{R}^{m\times n}$ and $B$ of arbitrary dimensions,

\centering
$A \otimes B =
\begin{bmatrix}
[A]_{1,1} B & \cdots & [A]_{1,n} B \\
\vdots & \ddots & \vdots \\
[A]_{m,1} B & \cdots & [A]_{m,n} B
\end{bmatrix}$
\end{definition}

\begin{definition}[Vectorize operator]
$\operatorname{vec}(X)$ stacks the columns of $X\in\mathbb{R}^{m\times n}$ into a vector of $\mathbb{R}^{mn}$.
\end{definition}

With column-major $\operatorname{vec}$, the following are exact: $\operatorname{vec}(g_t)=\operatorname{vec}(\delta_tx_t^\top)=x_t\otimes\delta_t$, and $\operatorname{vec}(g_t)\operatorname{vec}(g_s)^\top=(x_tx_s^\top)\otimes(\delta_t\delta_s^\top)$.

The Kronecker-factored form rests on three approximations, which we state separately rather than fold into the derivation.

\begin{assumption}[Layer-wise block diagonality]
\label{as:block}
$F$ is approximated by its block-diagonal part over layers, i.e.\ cross-layer parameter interactions are discarded~\cite{koroko2023efficient, sweeney2026geometry}.
\end{assumption}

\begin{assumption}[Independent activations and derivatives]
\label{as:iad}
$\mathbb{E}[(x_tx_s^\top)\otimes(\delta_t\delta_s^\top)]\approx\mathbb{E}[x_tx_s^\top]\otimes\mathbb{E}[\delta_t\delta_s^\top]$ for all $t,s$.
\end{assumption}
\fix{The proof asserted ``assuming $x_t\perp\delta_s$'' and then wrote an \emph{equality}. Statistical independence of activations and backpropagated errors is false in any trained network ($\delta$ is a function of $x$ through the forward pass); this is the standard IAD \emph{approximation} of K-FAC and the derivation must carry $\approx$, not $=$, from here on.}

\begin{assumption}[Temporal homogeneity of the input second moment]
\label{as:rate}
$\mathbb{E}[x_tx_s^\top]\approx A$ for all $t,s$, where $A$ is the pairwise average
\begin{equation}
A:=\frac{1}{\tau^2}\sum_{t,s=1}^{\tau}\mathbb{E}[x_tx_s^\top]=\mathbb{E}\big[rr^\top\big],\qquad r:=\frac{1}{\tau}\sum_{t=1}^{\tau}x_t,
\end{equation}
i.e.\ the curvature sees the spike train through its \emph{rate} vector $r$ rather than through its per-timestep correlation structure.
\end{assumption}

\fix{This was justified by ``$x_t$ is i.i.d.\ across timesteps, hence $\mathbb{E}[x_tx_s^\top]=A$ for all $t,s$''. Both halves fail. (i) Spike trains are not i.i.d.\ over $t$: encoders and the membrane recurrence induce temporal correlation, which is the property the paper elsewhere claims to exploit. (ii) Even for genuinely i.i.d.\ $x_t$ the conclusion is false: $\mathbb{E}[x_tx_t^\top]=\Sigma+\mu\mu^\top$ while $\mathbb{E}[x_tx_s^\top]=\mu\mu^\top$ for $t\neq s$, so diagonal and off-diagonal blocks differ by $\Sigma$. Constancy over $(t,s)$ is an independent modelling approximation, and the only definition of $A$ consistent with the estimator actually used in the algorithm is the pairwise average above.}

\begin{lemma}[Kronecker-factored Fisher for a shared weight]
\label{lem:fim-approx}
Let $W$ be shared across $t=1,\dots,\tau$, with $g_t=\delta_tx_t^\top$. Under Assumptions~\ref{as:block}--\ref{as:rate},
\begin{equation}
    F_W \approx A\otimes G,\quad
    A = \mathbb{E}[rr^\top],\quad
    G = \mathbb{E}\!\Big[\big(\textstyle\sum_{t=1}^\tau\delta_t\big)\big(\textstyle\sum_{t=1}^\tau\delta_t\big)^{\!\top}\Big],
\end{equation}
and, whenever $A$, $G$ are non-singular, $(A\otimes G)^{-1}=A^{-1}\otimes G^{-1}$, so the preconditioned update is $\Delta W=-\gamma\,G^{-1}(\nabla_W\mathcal{L})A^{-1}$.
\end{lemma}

\begin{proof}
By Definition~\ref{def:fisher} and $\nabla_W\log p_\theta=\sum_t g_t$, $F_W=\sum_{t,s=1}^\tau\mathbb{E}[\operatorname{vec}(g_t)\operatorname{vec}(g_s)^\top]=\sum_{t,s=1}^\tau\mathbb{E}[(x_tx_s^\top)\otimes(\delta_t\delta_s^\top)]$, the last step being the exact mixed-product identity. Assumption~\ref{as:iad} gives $F_W\approx\sum_{t,s}\mathbb{E}[x_tx_s^\top]\otimes\mathbb{E}[\delta_t\delta_s^\top]$. Only now, with the left factor constant in $(t,s)$ by Assumption~\ref{as:rate}, may the double sum be pulled through $\otimes$ by bilinearity:
\begin{equation}
F_W\approx A\otimes\!\!\sum_{t,s=1}^\tau\!\mathbb{E}[\delta_t\delta_s^\top]=A\otimes\mathbb{E}\big[(\textstyle\sum_t\delta_t)(\textstyle\sum_t\delta_t)^\top\big].
\end{equation}
Finally $\operatorname{vec}(BXC)=(C^\top\otimes B)\operatorname{vec}(X)$ with $A,G$ symmetric gives $F_W^{-1}\operatorname{vec}(\nabla_W\mathcal{L})=\operatorname{vec}(G^{-1}(\nabla_W\mathcal{L})A^{-1})$.
\end{proof}
\fix{The previous proof moved the double sum through $\otimes$ ``by bilinearity'' \emph{before} establishing that one factor is constant in $(t,s)$. That step is invalid in general: $\sum_{t,s}A_{ts}\otimes G_{ts}\neq(\sum_{t,s}A_{ts})\otimes(\sum_{t,s}G_{ts})$. It is Assumption~\ref{as:rate}, not bilinearity, that does the work, which is why the assumption cannot be presented as a consequence of i.i.d.\ sampling. The old proof also asserted $F_W^{-1}=(A\otimes G)^{-1}$ with no invertibility condition, while $A=\mathbb{E}[rr^\top]$ is routinely rank-deficient for sparse spike inputs; the inverse exists only after damping (Remark~\ref{rem:damping}).}

\fix{The update was written $(G+\epsilon I)^{-1}\nabla(A+\epsilon I)^{-1}$ and presented as the inverse of the damped Fisher, which it is not; undamped $\epsilon$ split evenly across factors also makes the effective damping scale-dependent, since $A\otimes G$ is invariant to $A\to cA,\,G\to G/c$ but the damped product is not.}

\begin{remark}[Scope of the derivation]
\label{rem:conv}
Lemma~\ref{lem:fim-approx} is derived for a fully-connected weight $W$ shared over time. For the conv-layers used in Sec.~\ref{sec:experiment}, $x_t$ is understood as the patch-extracted input and the expectations additionally average over spatial locations, i.e.\ the spatially-uncorrelated-derivatives approximation of KFC is assumed on top of Assumptions~\ref{as:block}--\ref{as:rate}.
\end{remark}

\begin{theorem}[\approach~ complexity]
\label{thm:complexity}
Let layer $\ell$ have input/output dimensions $m_\ell,n_\ell$, let $N_\ell=m_\ell n_\ell$ and $N=\sum_{\ell=1}^L N_\ell$, with batch size $B$ and $\tau$ timesteps. One \approach~ update costs
\begin{equation}
    O\Big(\sum_{\ell=1}^L\Big[B\tau(m_\ell{+}n_\ell)+B(m_\ell^2{+}n_\ell^2)+\tfrac{m_\ell^3+n_\ell^3}{T_{\mathrm{inv}}}+N_\ell(m_\ell{+}n_\ell)\Big]\Big).
\end{equation}
The curvature terms satisfy $\sum_\ell(m_\ell^3+n_\ell^3)$

$\le 2\sum_\ell N_\ell^3/\min(m_\ell,n_\ell)^3\le 2N^3/\min_\ell\min(m_\ell,n_\ell)^3$, so inverting the Kronecker factors is cheaper than the $O(N^3)$ inversion of the full Fisher by at least a factor $\tfrac12\min_\ell\min(m_\ell,n_\ell)^3$; for balanced layers ($m_\ell\asymp n_\ell$) the per-layer curvature cost is $O(N_\ell^{3/2})$ instead of $O(N_\ell^{3})$.
\end{theorem}

\begin{proof}
Per layer: $R^{(\ell),b},D^{(\ell),b}$ cost $O(B\tau(m_\ell+n_\ell))$; the two rank-one accumulations cost $O(B(m_\ell^2+n_\ell^2))$; inverting (or eigendecomposing) $A^{(\ell)},G^{(\ell)}$ costs $O(m_\ell^3+n_\ell^3)$ once every $T_{\mathrm{inv}}$ steps; the two matrix products forming $\Delta\theta^{(\ell)}$ cost $O(m_\ell n_\ell(m_\ell+n_\ell))$. Summing over $\ell$ gives the bound. For the comparison, $m^3+n^3\le 2\max(m,n)^3=2(mn)^3/\min(m,n)^3$ for any $m,n\ge1$, and $\sum_\ell N_\ell^3\le(\sum_\ell N_\ell)^3=N^3$ since the $N_\ell$ are non-negative. If $m_\ell=n_\ell=d_\ell$ then $N_\ell=d_\ell^2$ and $m_\ell^3+n_\ell^3=2N_\ell^{3/2}$.
\end{proof}
\fix{The old proof argued $N=\sum_\ell m_\ell n_\ell\ \ge\ \sum_\ell\max(m_\ell,n_\ell)^2\ \ge\ \sum_\ell(m_\ell^3+n_\ell^3)^{2/3}$. The first inequality is backwards ($m_\ell n_\ell\le\max(m_\ell,n_\ell)^2$) and the second is not a valid comparison at all; the chain also concluded nothing about $N^3$ versus the stated bound. The cost figure itself omitted the $O(N_\ell(m_\ell+n_\ell))$ preconditioning products and the $T_{\mathrm{inv}}$ amortization, and its ``$O(B\tau(m^2+n^2))$'' was inconsistent with the $O(\tau^2)$ form of $\hat A$ in the algorithm.}

\section{Experiment}
\label{sec:experiment}

\noindent\textbf{Baselines.} We compare \approach~ against several popular optimizers in SNN optimization, such as SGD with momentum~\cite{robbins1951stochastic}, Adam~\cite{kingma2014adam}, and AdamW~\cite{loshchilov2017decoupled}.

\begin{figure}[t]
    \centering
    \includegraphics[width=0.8\linewidth]{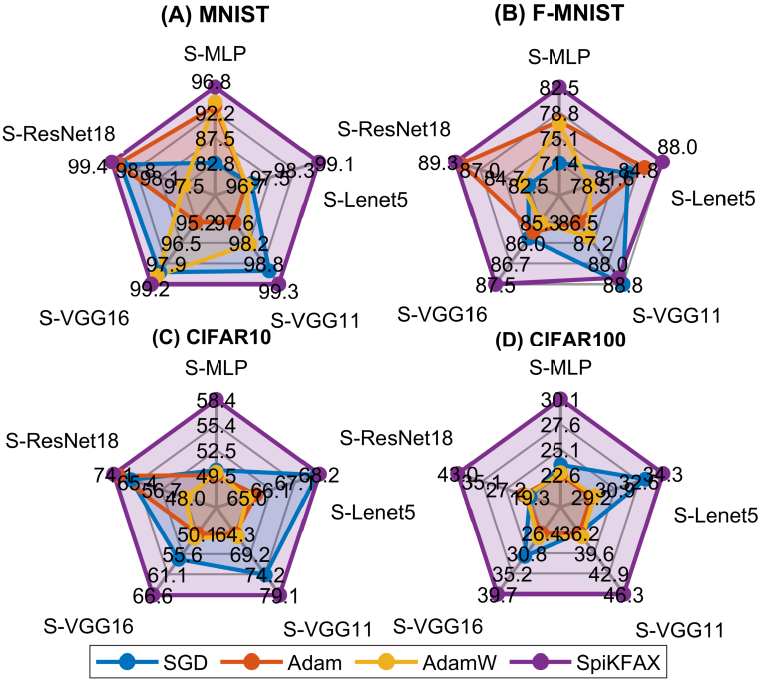}
    \vspace{-4mm}\caption{Performance comparison of different optimizers and models across static image datasets.}
    \label{fig:static_data}
\end{figure}

\noindent\textbf{Datasets and model architectures.} We study our proposed optimizer over diversed datasets such as $\textit{Static Datasets}$ (MNIST~\cite{deng2012mnist}, F-MNIST~\cite{xiao2017fashion}, CIFAR10~\cite{krizhevsky2009learning}, CIFAR100)~\cite{krizhevsky2009learning}, $\textit{Neuromorphic}$ $
\textit{Datasets}$ (N-MNIST~\cite{orchard2015converting}, CIFAR10-DVS~\cite{li2017cifar10}, DVS128 Gesture). Several architectures are exploited, such as spiking MLP, spiking LeNet5~\cite{lecun1998gradient}, spiking VGG11, spiking VGG16~\cite{simonyan2014very}, and spiking ResNet18~\cite{he2016deep}.

\noindent\textbf{Training settings.} Each experiment is conducted five times, and the mean values are reported. For each optimizer, we vary the values of the learning rate and choose the best setting. We conduct experimental implementation using the SNNTorch framework on a 64 GB RAM, 20-core CPU @NVIDIA RTX A5000 PC. 
Our code is available in our \href{https://github.com/ngocphucck/SpiKFAX}{repository}.

\noindent\textbf{Generalization comparison.} Across multiple datasets and model architectures, our results demonstrate that \approach~ significantly outperforms baseline optimizers in terms of testing accuracy. Specifically, Table \ref{tab:neuromorphic_results} compares the generalization of models optimized by different methods on three neuromorphic datasets: N-MNIST, CIFAR10-DVS, and DVS128-Gesture. On each dataset, the highest accuracy model is consistently the one trained with \approach, achieving $99.92\%$, $56.15\%$, and $76.13\%$, respectively. The generalization gains are substantial across both datasets and architectures: for instance, S-VGG11 trained with \approach~ improves testing accuracy by 1.87\% over the same architecture trained AdamW, while the improvements on CIFAR10-DVS and DVS128-Gesture are even larger, reaching at least 3.5\% and 6.43\%. This generalization benefit is extended to the traditional static image datasets as well, including MNIST, F-MNIST, CIFAR10, and CIFAR100, as shown in Fig. \ref{fig:static_data}. Across all five model architectures evaluated, training with \approach~ consistently yields a significant margin in testing accuracy over other optimizers. 

\begin{figure}[t]
    \centering
    \includegraphics[width=\linewidth]{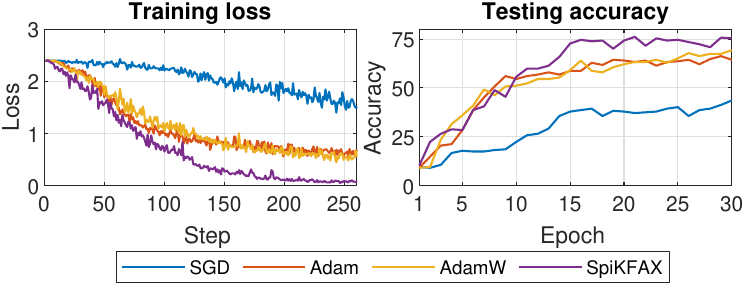}
    \vspace{-6mm}\caption{Training loss (over steps) and testing accuracy (over epochs) of S-VGG11 on the DVS128-Gesture dataset.}
    \label{fig:loss}
\end{figure}

\noindent\textbf{The stability of \approach.} Fig. \ref{fig:loss} shows the training loss and the testing accuracy of S-VGG11 across different optimizers. \approach~ drives a dramatically steeper reduction in the training loss, a trend that becomes clearly visible from the $75^{th}$ step onward. Moreover, the model trained with \approach~ exhibits some fluctuation at early steps before stabilizing, while models trained with other optimizers remain unstable throughout training. In terms of testing accuracy, the model optimized with \approach~ rapidly reaches its peak accuracy of 76.13\% by the $15^{th}$ epoch. In contrast, models trained with Adam and AdamW improve more gradually and plateau at lower accuracies of 67.42\% and 69.7\%, respectively, while the model trained with SGD performs worst overall. These results show that \approach~ not only stabilizes training but also achieves superior generalization.

\noindent\textbf{Ablation study.} We study the effect of the learning rate on the performance of models trained with different optimizers (c.f. Fig. \ref{fig:lr} A). When varying the learning rate, the model optimized by \approach~ stays stable with high accuracy. The generalization of the model trained by \approach~ consistently performs better than optimization by other methods across different learning rates. 

\noindent\textbf{Training time.} We measure the training time of \approach~ and compare it against other optimizers, including SGD, Adam, and AdamW (Fig. \ref{fig:lr} B). For each iteration, our algorithm is slower than SGD, Adam, and AdamW by 1.4x, 1.2x, and 1.2x, respectively. However, the exact Fisher information computation-based optimization is intractable even for the MLP model, i.e., the running time tends to infinity. Our approximation algorithm reduces the running time significantly and is comparable to second-order optimizers such as Adam or AdamW, while achieving higher generalization.

\begin{figure}[t]
    \centering
    \includegraphics[width=\linewidth]{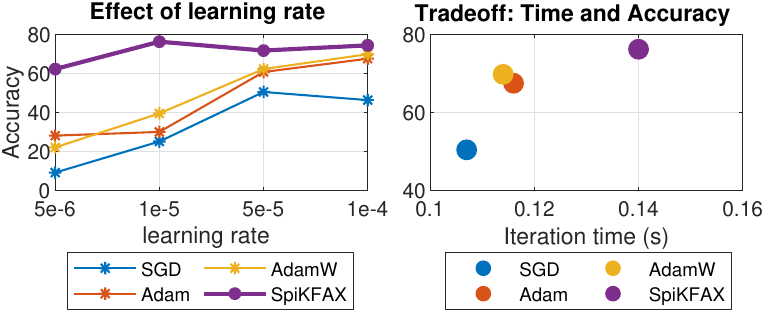}
    \vspace{-6mm}\caption{Effect of learning rate (left) and training time (right) of S-VGG11 on the DVS128-Gesture dataset.}
    \label{fig:lr}
\end{figure}

\section{Related Work}
\label{sec:related_works}

Several works design optimization algorithms specifically for the spiking setting. For instance, the paper \cite{windhager2026linearized} pairs Linearized Bregman Iterations with AdaBreg, a Bregman variant of Adam, to enforce weight sparsity during training. Other works target the sharpness of the SNN loss landscape directly, applying sharpness-aware minimization \cite{nicholson2026sharpnessawaresurrogatetrainingonsensor} to encourage flatter, better-generalizing solutions. A separate line departs from backpropagation-through-time altogether, replacing it with local or feedback-driven learning rules; a feedback control optimizer \cite{saponati2025feedback}, for instance, integrates spike-based local learning with control signals for online, on-chip training without storing intermediate states. These works improve SNN training via the objective function, sparsity, or the credit-assignment mechanism, but none exploit curvature information through a Kronecker-factored Fisher approximation which is the gap \approach~ addresses.

\section{Conclusion}
\label{sec:conclusion}

We present \approach, a second-order optimizer adapting K-FAC to spiking neural networks by deriving a Kronecker-factored Fisher approximation tailored to their time-recurrent, surrogate-gradient dynamics while remaining computationally tractable. \approach~ shows that the sharp loss landscape long blamed for slow, poorly generalizing SNN training can be addressed directly through tractable second-order curvature approximation, rather than only through indirect fixes like sharpness perturbation or sparsity constraints.

\newpage
\section{Acknowledgement}
This work is partially funded by the UK Government through the New Deal for Northern Ireland. The funding is delivered on behalf of the Northern Ireland Office and the Department for Digital, Culture, Media and Sport by Innovate UK. It is in part  supported by the CHIST-ERA grant through EPSRC Grant EP/Y03631X/1. 
\bibliographystyle{IEEEbib}
\bibliography{refs}

\end{document}